\documentclass[10pt,journal]{IEEEtran}

\usepackage{mathtools, amsthm, amsopn, amsmath}
\usepackage{etex}
\usepackage{graphicx}
\usepackage{endnotes}
\usepackage{hyperref}
\usepackage{epsfig,psfrag}
\usepackage{pst-all}
\usepackage{amssymb,amsfonts,upref,cite,epsf,color,bm}
\usepackage{graphicx}
\usepackage{color}
\usepackage{calc}
\usepackage{booktabs}
\usepackage{tikz}
\usepackage{pgfplots}
\newcommand\defeq{:=}
\usepackage{subfig}

\usepackage{algorithm, float}
\usepackage{algpseudocode}
\floatname{algoirthm}{algorithm}
\algnewcommand\algorithmicinput{\textbf{Input:}}
\algnewcommand\INPUT{\item[\Alg.\icinput]}
\algnewcommand\algoirhtmicoutput{\textbf{Output:}}
\algnewcommand\OUTPUT{\item[\algorithmicoutput]}

\DeclareMathOperator*{\argmin}{argmin}

\newcommand\vect[1]{\mathbf #1}
\newcommand{\vu}{\vect{u}}

\newcommand{\gsignal}{\mathbf{w}}
\newcommand{\gdual}{\mathbf{u}}
\newcommand{\hatgsignal}{\widehat{\gsignal}}
\newcommand{\numnodes}{N}
\newcommand{\numedges}{E}
\newcommand{\glabel}{y}
\newcommand{\gindex}[1][i]{^{(#1)}}
\newcommand{\nodeidx}{i}
\newcommand{\gweight}{A}
\newcommand{\sigdimens}{d}
\newcommand{\gfeature}{\mathbf{x}}
\newcommand{\gvariable}{\mathbf{v}}
\newcommand{\gvariablep}{\mathbf{v}'}

\newcommand{\incidence}{\vect{D}}
\newcommand{\samplesize}{M}

\newcommand{\samplingset}{\mathcal{M}}
\newcommand{\edges}{\mathcal{E}}
\newcommand{\nodes}{\mathcal{V}}
\newcommand{\graph}{\mathcal{G}}
\newcommand{\graphsigs}{\mathcal{C}} 

\newcommand{\diag}{{\rm diag}}

\newtheorem{theorem}{Theorem}

\newtheorem{lemma}[theorem]{Lemma}

\definecolor{lavander}{cmyk}{0,0.48,0,0}
\definecolor{violet}{cmyk}{0.79,0.88,0,0}
\definecolor{burntorange}{cmyk}{0,0.52,1,0}

\def\oran{orange!30}

\tikzstyle{vertex}=[draw,circle,burntorange, left color=\oran,
                       text=blue,fill=black!25,minimum size=15pt,inner sep=0pt]
\tikzstyle{sampledvertex}=[draw,circle,black, 
                       text=blue,fill=black!25,minimum size=15pt,inner sep=0pt]                       

\tikzstyle{legendvertex}=[rectangle, right,
                       text=black,minimum size=15pt,inner sep=0pt]

\tikzstyle{legendsample}=[rectangle, right,
                       text=black,minimum size=15pt,inner sep=0pt]  

\title{Classifying Partially Labeled Networked Data via Logistic Network Lasso} 

\author{Nguyen Tran, Henrik Ambos and Alexander Jung

\thanks{Authors are with the Department of Computer Science, Aalto University, Finland; firstname.lastname(at)aalto.fi}}

\begin{document}
	\maketitle
\begin{abstract}
We apply the network Lasso to classify partially labeled data points which are characterized by high-dimensional feature vectors. 
In order to learn an accurate classifier from limited amounts of labeled data, we borrow statistical strength, via an intrinsic network structure, 
across the dataset. The resulting logistic network Lasso amounts to a regularized empirical risk minimization problem using the total variation 
of a classifier as a regularizer. This minimization problem is a non-smooth convex optimization problem which we solve using a primal-dual 
splitting method. This method is appealing for big data applications as it can be implemented as a highly scalable message passing algorithm. 
	
\end{abstract}


\section{Introduction}
	The \emph{least absolute shrinkage and selection operator} (Lasso) has been extended to networked data recently. 
	This extension, coined the ``network Lasso'' (nLasso), allows efficient processing of massive datasets using convex optimization methods \cite{NetworkLasso}. 
	
	Most of the existing work on nLasso-based methods focuses on predicting numeric labels (or target variables) within regression 
	problems \cite{WhenIsNLASSO, LocalizedLasso, NNSPFrontiers,NetworkLasso,Chen2015,SandrMoura2014,JungSLPcomplexit2018}. 
	In contrast, we apply nLasso to binary classification problems which assign binary-valued 
	labels to data points \cite{Bhagat11,LargeGraphsLovasz,BigDataNetworksBook}. 
	
	In order to learn a classifier from partially labeled networked data, we minimize the logistic loss incurred on a training set constituted by 
	few labeled nodes. Moreover, we aim at learning classifiers which conform to the intrinsic network structure of the data. In particular, 
	we require classifiers to be approximately constant over well-connected subsets (clusters) of data points. This cluster assumption lends naturally to regularized 
	empirical risk minimization with the total variation of the classifier as regularization term \cite{VapnikBook}. We solve this non-smooth convex 
	optimization problem by applying the primal-dual method proposed in \cite{pock_chambolle, Condat2013}.
	 
	The proposed classification method extends the toolbox for semi-supervised classification in networked data \cite{SemiSupervisedBook, Zhan2014, Ruusuvuori2012,Kechichian2018,Boykov2004}. 
	In contrast to label propagation (LP), which is based on the squared error loss, we use the logistic loss which is more suitable for classification problems. 
	Another important difference between LP methods and nLasso is the different choice of regularizer. Indeed, LP uses the Laplacian quadratic form 
	while the nLasso uses total variation for regularization. 
	
	Using a (probabilistic) stochastic block model for networked data, a semi-supervised classification method is obtained as an instance of belief propagation 
	method for inference in graphical models \cite{Zhan2014}. In contrast, we assume the data (network-) structure as fixed and known. The proposed method 
	provides a statistically well-founded alternative to graph-cut methods \cite{Ruusuvuori2012,Kechichian2018,Boykov2004}. While our approach is based on a 
	convex optimization (allowing for highly scalable implementation), graph-cuts is based on combinatorial optimization which makes scaling them to large datasets 
	more challenging. Moreover, while graph-cut methods apply only to data which is characterized by network structure and labels, our method allows to exploit additional 
	information provided by feature vectors of data points. 
	
	\textbf{Contribution:} Our main contributions are: (i) We present a novel implementation of logistic network Lasso by applying a primal-dual method. This method can be 
	implemented as highly scalable message passing on the network structure underlying the data. (ii) We prove the convergence of this primal-dual method and (iii) 
	verify its performance on synthetic classification problems in chain and grid-structured data.

	\textbf{Notation:} 
Boldface lowercase (uppercase) letters denote vectors (matrices). We denote $\vect{x}^T$ the transpose of vector $\vect{x}$. The $\ell_2$-norm of a vector $\vect{x}$ is $\|\vect{x} \| = \sqrt{\vect{x}^T \vect{x}}$. The convex conjugate of a function $f$ is defined as $f^*(\vect{y}) = \sup_{\vect{x}} (\vect{y}^T \vect{x} - f(\vect{x}))$.  We also need the sigmoid function $\sigma(z) \defeq 1/(1+\exp(-z))$.
	
\section{Problem Formulation}
\label{sec_problem_formuation} 

	We consider networked data that is represented by an undirected weighted graph (the 
	``empirical graph'') $\graph = (\nodes, \edges, \mathbf{A})$. A particular node $i \in \nodes =  \{1, \ldots, \numnodes\}$ of the graph represents an 
	individual data point (such as a document, or a social network user profile).\footnote{With a slight abuse of notation, we refer by $i \in \nodes$ to a 
	node of the empirical graph as well as the data point which is represented by that node.} Two different data points $i,j \in \nodes$ 
	are connected by an undirected edge $\{i,j\} \in \edges$ if they are considered similar (such as documents authored by the same person or 
	 social network profiles of befriended users).  For ease of notation, we denote the edge set $\edges$ by $\{1, \ldots, \numedges\defeq|\edges|\} $. 
	
	Each edge $e = \{i,j\} \in \edges$ is endowed with a positive weight $\gweight_e = \gweight_{ij} > 0$ which quantifies the amount of similarity 
	between data points $i,j \in \nodes$. The neighborhood of a node $i \in \nodes$ is $\mathcal{N}(i) \defeq \{ j : \{i,j\} \in \edges \}$.
	
	Beside the network structure, datasets convey additional information in the form of features $\gfeature\gindex \in \mathbb{R}^{\sigdimens}$ and 
	labels $\glabel\gindex \in \{-1,1\}$ associated with each data point $i \in \nodes$. In what follows, we assume the features to be normalized such 
	that $\| \gfeature\gindex \| = 1$ for each data points $i \in \nodes$. While features are typically available for each data point $i \in \nodes$, 
	labels are costly to acquire and available only for data points in a small training set $\samplingset =\{i_{1},\ldots,i_{\samplesize}\}$ containing $\samplesize$ labeled data points.
	
	We model the labels $\glabel\gindex$ of the data points $i \in \nodes$ as independent random variables with (unknown) 
	probabilities 
	\begin{equation} 
	\label{equ_def_p_i}
	p\gindex \defeq {\rm Prob} \{ \glabel\gindex \!=\! 1 \} = \frac{1}{1+ \exp(-(\gsignal\gindex)^T \gfeature\gindex)}.
	\end{equation} 
	The probabilities $\{ p\gindex \}_{\nodeidx \in \nodes}$ are parametrized by some (unknown) weight vectors $\gsignal\gindex$. Our goal is to develop a method for learning an 
	accurate estimate $\widehat{\gsignal}\gindex$ of the weight vector $\gsignal\gindex$. Given the estimate $\widehat{\gsignal}\gindex$, we can compute
	an estimate $\hat{p}\gindex$ for the probability $p\gindex$ by replacing ${\gsignal}\gindex$ with $\widehat{\gsignal}\gindex$ in \eqref{equ_def_p_i}.
	
	We interpret the weight vectors as the values of a graph signal $\gsignal: \nodes \rightarrow \mathbb{R}^{\sigdimens}$
	assigning each node $i \in \nodes$ of the empirical graph $\graph$ the vector $\gsignal\gindex\in \mathbb{R}^{\sigdimens}$. 
	The set of all vector-valued graph signals is denoted $\graphsigs\!\defeq\!\{ \gsignal: \nodes \rightarrow \mathbb{R}^{\sigdimens}: i \mapsto \gsignal \gindex \}$. 
	
	Each graph signal $\widehat{\gsignal} \in \graphsigs$ defines a classifier which maps a node with features $\gfeature \gindex$ to the predicted label
	\begin{equation} 
	\widehat{\glabel}\gindex = \begin{cases} 1 &\text{if }  \big(\widehat{\gsignal}\gindex\big)^T \gfeature\gindex > 0 \\ 
	-1 & {\rm otherwise}.
	\end{cases} 
	\label{equ_label_assign}
	\end{equation}
	
	Given partially labeled networked data, we aim at leaning a classifier $\widehat{\gsignal} \in \graphsigs$ which agrees with the labels $\glabel\gindex$ 
	of labeled data points in the training set $\samplingset$. In particular, we aim at learning a classifier having a small training error 
	\begin{equation} 
	\label{equ_def_emp_risk}
	 \widehat{E}(\widehat{\gsignal})\!\defeq\!({1}/{\samplesize}) \sum_{i \in \samplingset} \!\ell ((\widehat{\gsignal}\gindex)^T \widetilde{\gfeature}\gindex) 
	\end{equation} 
	with $\widetilde{\gfeature}\gindex \defeq \glabel\gindex \gfeature\gindex$ and the logistic loss  
		\vspace*{-1mm}
	\begin{equation}
	\label{equ_def_log_loss}
	\ell (z) \defeq {\rm log} (1 + \exp(-z) ) = -{\rm log}(\sigma(z)).
	\vspace*{-3mm}
	\end{equation}
	
\section{Logistic Network Lasso}
\label{sec_lNLasso}
	
	The criterion \eqref{equ_def_emp_risk} by itself is not enough for guiding the learning of a classifier $\gsignal$ since \eqref{equ_def_emp_risk} completely
	ignores the weights $\widehat{\gsignal}\gindex$ at unlabeled nodes $i \in \nodes \setminus \samplingset$. Therefore, we need to impose some additional structure on the classifier $\widehat{\gsignal}$. 
	In particular, any reasonable classifier $\widehat{\gsignal}$ should conform with the \emph{cluster structure} of the empirical graph $\graph$ \cite{NewmannBook}. 
	
	We measure the extend of a classifier $\widehat{\gsignal} \in \graphsigs$ conforming with the cluster structure of $\graph$ by the total variation (TV)
	\vspace*{0mm}
	\begin{align}
	\label{equ_def_TV_norm}
		\| \gsignal \|_{\rm TV} & \defeq \sum_{\{i,j\}\in \edges} \gweight_{ij} \| \gsignal\gindex[j] - \gsignal\gindex \|. 
	\end{align}
	A classifier $\widehat{\gsignal}\!\in\!\graphsigs$ has small TV if the weights $\widehat{\gsignal}\gindex$ are approximately constant 
	over well connected subsets (clusters) of nodes. 
	
	We are led quite naturally to learning a classifier $\widehat{\gsignal}$ 
	via the  \emph{regularized empirical risk minimization} (ERM)
	\begin{align} \label{optProb}
		\hatgsignal & \in \argmin_{\gsignal \in \graphsigs} \widehat{E}(\gsignal)  + \lambda \| \gsignal \|_{\rm TV}. 
	\end{align}
	We refer to \eqref{optProb} as the logistic nLasso (lnLasso) problem. 
	The parameter $\lambda$ in \eqref{optProb} allows to trade-off small TV  $\| \hatgsignal \|_{\rm TV}$ 
	against small error $\widehat{E}(\hatgsignal)$ (cf. \eqref{equ_def_emp_risk}). 
        The choice of $\lambda$ can be guided by cross validation \cite{hastie01statisticallearning}. 	
	
	Note that lnLasso \eqref{optProb} does not enforce directly the labels $\glabel\gindex$ to be clustered. Instead, it requires the classifier $\widehat{\gsignal}$, which parametrizes the probability distributed of 
	the labels $\glabel\gindex$ (see \eqref{equ_def_p_i}), to be clustered. 
	
	It will be convenient to reformulate \eqref{optProb} using vector notation. We represent a graph signal $\gsignal \in \graphsigs$ as the vector
	\begin{align}
	\label{equ_def_vector_signal}
	\gsignal = ((\gsignal\gindex[1])^T, \ldots ,(\gsignal\gindex[\numnodes])^T)^T \in \mathbb{R}^{\sigdimens\numnodes}.
	\end{align}
	Define a partitioned matrix $\incidence \in \mathbb{R}^{(\sigdimens\numedges) \times (\sigdimens\numnodes)}$  block-wise as \begin{align}
	\incidence_{e,i} = \begin{cases}
	\gweight_{ij} \mathbf{I}_{\sigdimens} & e=\{i,j\}, i<j\\
	-\gweight_{ij} \mathbf{I}_{\sigdimens}& e=\{i,j\}, i>j\\
	\mathbf{0} & {\rm otherwise},
	\end{cases}
	\label{equ_def_incident_matrix}
	\end{align}
	where $\mathbf{I}_{\sigdimens} \in \mathbb{R}^{\sigdimens \times \sigdimens}$ is the identity matrix.
	The term $\gweight_{ij} (\gsignal\gindex - \gsignal\gindex[j])$ in \eqref{equ_def_TV_norm} is the $e$-th block of $\incidence \gsignal$. Using \eqref{equ_def_vector_signal} and \eqref{equ_def_incident_matrix}, we can reformulate the lnLasso \eqref{optProb} as
	\begin{align}\label{LNLprob}
		\hatgsignal \in \argmin_{\gsignal \in \mathbb{R}^{\sigdimens\numnodes}}   h(\gsignal) + g(\incidence \gsignal),
	\end{align}
	with 
	\begin{align}
	\label{equ_def_opt_func}
	h(\gsignal) = \widehat{E}(\gsignal) \text{ and }  g(\gdual) \defeq \lambda \sum_{e=1}^{\numedges}\|\gdual^{(e)}\|
	\end{align}
	 with stacked vector $\gdual = (\gdual^{(1)},\ldots,\gdual^{(\numedges})) \in \mathbb{R}^{\sigdimens\numedges}$.

\section{primal-dual method}
\label{sec_lNLasso_ADMM}
The  lnLasso \eqref{LNLprob} is a convex optimization problem with a non-smooth objective function which rules out the use of gradient descent methods \cite{JungFixedPoint}. However, the objective function is highly structured since it is the sum of a smooth convex function $h(\gsignal)$ and a non-smooth convex function $g(\incidence \gsignal)$,
 which can be optimized efficiently when considered separately. This suggests to use a proximal splitting method \cite{Combettes2009, Connor2014, pock_chambolle} for solving \eqref{LNLprob}. One particular such method is the preconditioned primal-dual method \cite{PrecPockChambolle2011} which is based on reformulating the problem \eqref{LNLprob} as a saddle-point problem 
\begin{align}
\label{equ_pd_prob}
\min_{\gsignal \in \mathbb{R}^{\sigdimens \numnodes}} \max_{\gdual \in \mathbb{R}^{\sigdimens\numedges}} \gdual^T\incidence \gsignal  + h(\gsignal) - g^*(\gdual),
\end{align}
with the convex conjugate  $g^*$ of $g$ \cite{pock_chambolle}.

Solutions $(\hatgsignal, \widehat{\gdual})$ of \eqref{equ_pd_prob} are characterized by \cite[Thm 31.3]{RockafellarBook} 
\begin{align}
-\incidence^T \widehat{\gdual}  \in \partial h(\hatgsignal)\nonumber\\
\incidence \hatgsignal \in \partial g^*(\widehat{\gdual}).
\label{equ_pd_cond_1}
\end{align}
This condition is, in turn, equivalent to
\begin{align} 
\hatgsignal - \mathbf{T} \incidence^T \widehat{\gdual}  \in (\mathbf{I}_{\sigdimens\numnodes} +  \mathbf{T} \partial h) (\hatgsignal),\nonumber\\
\widehat{\gdual} + \boldsymbol{\Sigma} \incidence \hatgsignal \in (\mathbf{I}_{\sigdimens\numedges} + \boldsymbol{\Sigma} \partial g^*)(\widehat{\gdual}),
\label{equ_pd_cond_1_2}
\end{align}
with positive definite matrices $ \boldsymbol{\Sigma} \!\in\! \mathbb{R}^{\sigdimens \numedges \times \sigdimens\numedges}, \mathbf{T} \!\in\! \mathbb{R}^{\sigdimens\numnodes \times \sigdimens\numnodes}$. The matrices $\boldsymbol{\Sigma}, \mathbf{T}$ are design parameters whose choice will be detailed below. 
The condition \eqref{equ_pd_cond_1_2} lends naturally to the following coupled fixed point iterations \cite{PrecPockChambolle2011}
\begin{align}
\gsignal_{k+1} \!&=\! (\mathbf{I} \!+\! \mathbf{T} \partial h)^{-1} (\gsignal_{k} \!-\!\mathbf{T} \incidence^T \gdual_{k})  \label{equ_pd_upd_x} \\
\gdual_{k+1} \!&=\! (\mathbf{I} \!+\! \boldsymbol{\Sigma} \partial g^*)^{-1} (\gdual_{k} \!+\! \boldsymbol{\Sigma} \incidence (2\gsignal_{k+1} \!-\! \gsignal_{k})).
\label{equ_pd_upd_y}
\end{align}

The update \eqref{equ_pd_upd_y} involves the resolvent operator 
\begin{align}
\label{equ_def_prox}
\hspace*{-2mm}(\mathbf{I} \!+\! \boldsymbol{\Sigma} \partial g^*)^{-1} (\gvariable) \!=\! \argmin_{\gvariablep \in \mathbb{R}^{\sigdimens \numedges} } g^*(\gvariablep) \!+\! (1/2)\| \gvariablep \!-\! \gvariable\|^2_{\boldsymbol{\Sigma}^{-1}},
\end{align}
where  $\|\gvariable\|_{\boldsymbol{\Sigma}} \!\defeq\! \sqrt{\gvariable^T \Sigma \gvariable}$. 
The convex conjugate $g^*$ of $g$ (see \eqref{equ_def_opt_func}) can be decomposed as $g^*(\gvariable) = \sum_{e=1}^{\numedges} g_2^*(\gvariable^{(e)})$ 
with the convex conjugate $g_2^*$ of the scaled $\ell_2$-norm $\lambda \|.\|$. Moreover, since $\boldsymbol{\Sigma}$ is a block diagonal matrix, the $e$-th block 
of the resolvent operator $(\mathbf{I}_{\sigdimens\numedges} + \boldsymbol{\Sigma} \partial g^*)^{-1} (\gvariable)$ can be obtained by the Moreau decomposition as \cite[Sec. 6.5]{ProximalMethods}
\begin{align}
((\mathbf{I}_{\sigdimens\numedges} + \boldsymbol{\Sigma} \partial& g^*)^{-1} (\gvariable))^{(e)} \nonumber\\
&\stackrel{\eqref{equ_def_prox}}{=} \argmin_{\gvariablep \in \mathbb{R}^{\sigdimens } } g_2^*(\gvariablep) \!+\! (1/(2\sigma^{(e)})) \| \gvariablep \!-\! \gvariable^{(e)}\|^2 \nonumber\\
&= \gvariable^{(e)} \!-\! \sigma^{(e)} (\mathbf{I}_{\sigdimens} \!+\! (\lambda/ \sigma^{(e)})\partial \|.\|)^{-1} (\gvariable^{(e)}/\sigma^{(e)}) \nonumber\\
&= \begin{cases}
\lambda \gvariable^{(e)}/ \|\gvariable^{(e)}\| & \text{if } \|\gvariable^{(e)}\|> \lambda\\
\gvariable^{(e)} & {\rm otherwise},
\end{cases} \nonumber
\end{align}
where $(a)_{+} \!=\! \max\{a, 0\}$ for $a \in \mathbb{R}$. 

The update \eqref{equ_pd_upd_x} involves the resolvent operator $(\mathbf{I} + \mathbf{T} \partial h)^{-1}$ of $h$ 
(see \eqref{equ_def_emp_risk} and \eqref{equ_def_opt_func}), which does not have a closed-form solution. 
Choosing $\mathbf{T}\!=\! {\rm diag} \{\tau^{(\nodeidx)}\mathbf{I}_{\sigdimens}\}_{\nodeidx=1}^{\numnodes}$, we can 
solve \eqref{equ_pd_upd_x} approximately by a simple iterative method \cite[Sec. 8.2]{DistrOptStatistLearningADMM}. 
Setting$\overline{\gsignal}\!\defeq\!\gsignal_{k}\!-\!\mathbf{T} \incidence^T \gdual_{k}$, 
the update \eqref{equ_pd_upd_x} becomes
\begin{align}
\hspace*{-2mm}\gsignal\gindex_{k+1}\!\defeq\!\argmin_{\widetilde{\gsignal} \in \mathbb{R}^{\sigdimens}} 2 \ell (\widetilde{\gsignal}^T \widetilde{\gfeature}\gindex)\!+\!(\samplesize/\tau^{(i)})\! \|\widetilde{\gsignal} \!-\! \overline{\gsignal}\gindex\|^2.
\label{equ_pd_upd_x_exact}
\end{align}

If the matrices $\boldsymbol{\Sigma}$ and $\mathbf{T}$ satisfy
\begin{align}
\|\boldsymbol{\Sigma}^{1/2} \incidence  \mathbf{T}^{1/2}\|^2 <1,
\label{equ_pre_cond}
\end{align}
the sequences obtained from iterating \eqref{equ_pd_upd_x} and \eqref{equ_pd_upd_y} converge to a saddle point of 
the problem \eqref{equ_pd_prob} \cite[Thm. 1]{PrecPockChambolle2011}. The condition \eqref{equ_pre_cond} is satisfied 
for the choice $\boldsymbol{\Sigma} = \{ (1/(2\gweight_{e})) \mathbf{I}_{\sigdimens}\}_{e \in \edges}$ and 
$\{(\tau/d^{(i)})  \mathbf{I}_{\sigdimens}\}_{i\in \nodes}$, with node degree $d^{(i)} = \sum_{j \neq i } \gweight_{ij} $ 
and some $\tau < 1$ \cite[Lem. 2]{PrecPockChambolle2011}. 

Solving  \eqref{equ_pd_upd_x_exact} is equivalent to the zero-gradient condition
\begin{align}
{-\widetilde{\gfeature} \gindex \sigma(-(\gsignal\gindex)^T \widetilde{\gfeature}\gindex)} + (\samplesize/\tau^{(i)}) ({\gsignal\gindex} \!-\! \overline{\gsignal}\gindex) =0.
\label{equ_pd_upd_x_grad}
\end{align}
The solutions of \eqref{equ_pd_upd_x_grad} are fixed-points of the map
\begin{align}
\label{equ_def_phi}
{\bf \Phi} \gindex(\vu) = \overline{\gsignal}\gindex + (\tau^{(i)}/\samplesize){\widetilde{\gfeature} \gindex\sigma(-\vu^T \widetilde{\gfeature} \gindex)}.
\end{align}

\begin{lemma}
The mapping ${\bf \Phi} \gindex$ \eqref{equ_def_phi} is Lipschitz with constant $\beta_i = \tau^{(i)} \|\gfeature\gindex\|^2/\samplesize$.
\end{lemma}
\begin{proof}
For any $a, b \in \mathbb{R}$, 
\vspace*{-2mm}
\begin{align}
\big|1/({1+\exp(a)}) \!-\! {1}/({1\!+\!\exp(b)})\big| \leq |a\!-\!b|\nonumber
\end{align}
which implies
\begin{align}
\big|\sigma(-\vu^T \widetilde{\gfeature} \gindex) &\!-\! \sigma(-\mathbf{v}^T \widetilde{\gfeature} \gindex) \big| \leq \|\gfeature\gindex\| \| \mathbf{u}- \mathbf{v}\|, \nonumber
\end{align}
and, in turn,
\begin{align}
\|{\bf \Phi} \gindex(\vu) - {\bf \Phi} \gindex(\mathbf{v})\| &\stackrel{}{\leq} (\tau^{(i)} \| \widetilde{\gfeature} \gindex\|/\samplesize)\|\gfeature\gindex\| \| \mathbf{u}- \mathbf{v}\| \nonumber\\[0mm]
& = \beta_i \| \mathbf{u}- \mathbf{v}\|. \nonumber \\[-11mm]
\nonumber
\vspace*{-11mm}
\end{align}
\end{proof}
We approximate the exact update \eqref{equ_pd_upd_x_exact}  with 
\begin{align}
\widehat{\gsignal}_{k+1}\gindex = \underbrace{{\bf \Phi} \gindex \circ \ldots  \circ{\bf \Phi} \gindex }_{\lceil2\log(k)/\log(1/\beta_i)\rceil}(\overline{\gsignal}\gindex).
\label{equ_pd_upd_x_inexact}
\end{align}
According to \cite[Thm. 1.48]{BausckeCombette}, for $\tau^{(\nodeidx)}\!<\!\samplesize/\|\gfeature\gindex\|^2$ the error incurred by 
replacing \eqref{equ_pd_upd_x_exact} with \eqref{equ_pd_upd_x_inexact} satisfies 
\begin{align}
e_k = \|\widehat{\gsignal}_{k+1}\gindex - {\gsignal}_{k+1}\gindex \| \leq 1/k^2.
\label{equ_pd_upd_x_err}
\end{align}

Given the error bound \eqref{equ_pd_upd_x_err}, as can be verified using \cite[Thm. 3.2]{Condat2013},  the sequences 
obtained by \eqref{equ_pd_upd_x} and \eqref{equ_pd_upd_y} when replacing the exact update \eqref{equ_pd_upd_x_exact} 
with \eqref{equ_pd_upd_x_inexact} converge to a saddle-point of \eqref{equ_pd_prob} and, in turn, a solution of lnLasso \eqref{LNLprob}.

\begin{algorithm}[]
\caption{lnLasso via primal-dual method}\label{alg:ADMM}
\begin{algorithmic}[1]
\renewcommand{\algorithmicrequire}{\textbf{Input:}}
\renewcommand{\algorithmicensure}{\textbf{Output:}}
\Require   $\graph = (\nodes, \edges, \mathbf{A})$, $\{\gfeature\gindex\}_{\nodeidx \in \nodes}$, $\samplingset$, 
$\{ \glabel\gindex \}_{i \in \samplingset}$, $\lambda$, $\incidence$, $\boldsymbol{\Sigma} = \diag\{ \sigma^{(e)}=1/(2\gweight_{e})\mathbf{I}_{\sigdimens}\}_{e=1}^{\numedges}$, 
$\mathbf{T} =  \diag\{\tau^{(\nodeidx)} = 0.9/d^{(\nodeidx)}\mathbf{I}_{\sigdimens}\}_{i\in \nodes}$, $\beta_{\nodeidx} = \tau^{(i)} /|\samplingset|$
\Statex\hspace{-6mm}{\bf Initialize:} $k\!\defeq\!0$, $\widehat{\gsignal}_0\!\defeq\!0$, $\widehat{\gdual}_0\!\defeq\!0$
		\Repeat
			\State $\widehat{\gsignal}_{k+1} \defeq \widehat{\gsignal}_{k} - \mathbf{T} \incidence^T \widehat{\gdual}_{k}$
			\For{ each labeled node $\nodeidx \in \samplingset$}
				\State 
				\vspace{-4mm}\begin{align}
				\widehat{\gsignal}_{k+1}\gindex \defeq \underbrace{{\bf \Phi} \gindex \circ \ldots  \circ{\bf \Phi} \gindex}_{\lceil2\log(k)/\log(1/\beta_i)\rceil}(\widehat{\gsignal}_{k+1}\gindex) \nonumber
				\end{align}
			\EndFor	
			\State  $\overline{\gdual} \defeq {\gdual}_k + \boldsymbol{\Sigma} \incidence (2\widehat{\gsignal}_{k+1}-\widehat{\gsignal}_{k})$ 

			\State  $\widehat{\gdual}_{k+1}^{(e)} = \overline{\gdual}^{(e)} - \bigg(1- \frac{\lambda}{\|\overline{\gdual}^{(e)}\|}\bigg)_{+} \overline{\gdual}^{(e)}$ for $e \in \edges$

		\State  $k\!\defeq\!k\!+\!1$
		\Until stopping criterion is satisfied 
		\Ensure $(\hatgsignal_{k},\widehat{\gdual}_{k})$.
	\end{algorithmic}
\end{algorithm}

\section{Numerical Experiments}
\label{sec_num}

We assess the performance of lnLasso Alg.\ \ref{alg:ADMM} on datasets whose empirical graph is either a chain graph or a grid graph. 

\subsection{Chain}

For this experiment, we generate a dataset whose empirical graph is a chain consisting of $N\!=\!400$ nodes which represent 
individual data points. The chain graph is partitioned into 8 clusters, $\mathcal{C}_r = \{r\cdot50 +1, \ldots, r\cdot50 +50\}$, for $r = 0, \ldots 7$. 
The edge weights $\gweight_{ij}$ are set to $100$ if nodes $i$ and $j$ belong to the same cluster and $1$ otherwise. 

All nodes $\nodeidx \in \mathcal{C}_{r}$ in a particular cluster $\mathcal{C}_r$ share the same weight vector $\mathbf{w}^{(r)} \sim \mathcal{N}(0,\mathbf{I})$ 
which is generated from a standard normal distribution.  The feature vectors $\gfeature\gindex \in \mathbb{R}^{3}$ are 
generated i.i.d. using a uniform distribution over $[0,1]^3$. The true node labels $\glabel\gindex$ are drawn from the 
distribution \eqref{equ_def_p_i}. The training set is obtained by independently selecting each node with probability (labeling rate) $p$. 

We apply  Alg. \ref{alg:ADMM} to obtain a classifier $\widehat{\gsignal}$ which allows to classify data points as  
$\widehat{\glabel}\gindex ={\rm sign} ((\widehat{\gsignal}\gindex)^T \gfeature\gindex) $. In order to assess the performance of Alg. \ref{alg:ADMM} 
we compute the accuracy within the unlabeled nodes, i.e., the ratio of the number of correct labels achieved by Alg. \ref{alg:ADMM} for unlabeled nodes to the number of unlabeled nodes,
\vspace{-2mm}
\begin{align}
ACC \defeq \big(1/({\numnodes - \samplesize})\big) {|\{i : \glabel\gindex = \widehat{\glabel}{}\gindex, i \notin \samplingset\}|}. 
\vspace{-3mm}
\label{equ_exp_accu}
\end{align}
We compute the accuracy of Alg. \ref{alg:ADMM}  for  different choices of $p\in \{0.1, \ldots, 0.9\}$ and $\lambda\in\{10^{-5}, \ldots, 10^{-1}\}$. For a pair of $\{\lambda, p\}$, we repeat the experiment $100$ times and compute the average accuracy.

The accuracies obtained for varying labeling rates $p$ and lnLasso parameter $\lambda$ is plotted in Fig. \ref{fig:accuracy_chain_8}. 
As indicated in Fig. \ref{fig:accuracy_chain_8}, the accuracy increases with labeling rate $p$ which confirms the intuition that increasing 
the amount of labeled data (training set size) supports the learning quality. 

The accuracies in Fig. \ref{fig:accuracy_chain_8} are low since the classifier assigns the labels to nodes based on \eqref{equ_label_assign} while the true labels are drawn from the probabilistic model \eqref{equ_def_p_i}. Indeed, we  also plot in Fig. \ref{fig:accuracy_chain_8} the optimal accuracy, determined by the average of the probability to assign nodes $i$ to their true label when knowing $p\gindex$ using \eqref{equ_def_p_i}, as a horizontal dashed line. Fig. \ref{fig:accuracy_chain_8} shows that accuracies increase with labeling rate $p$ and tends toward the optimal accuracy.
\vspace{-2mm}
\begin{figure}[htbp]
	\includegraphics[width=1\columnwidth]{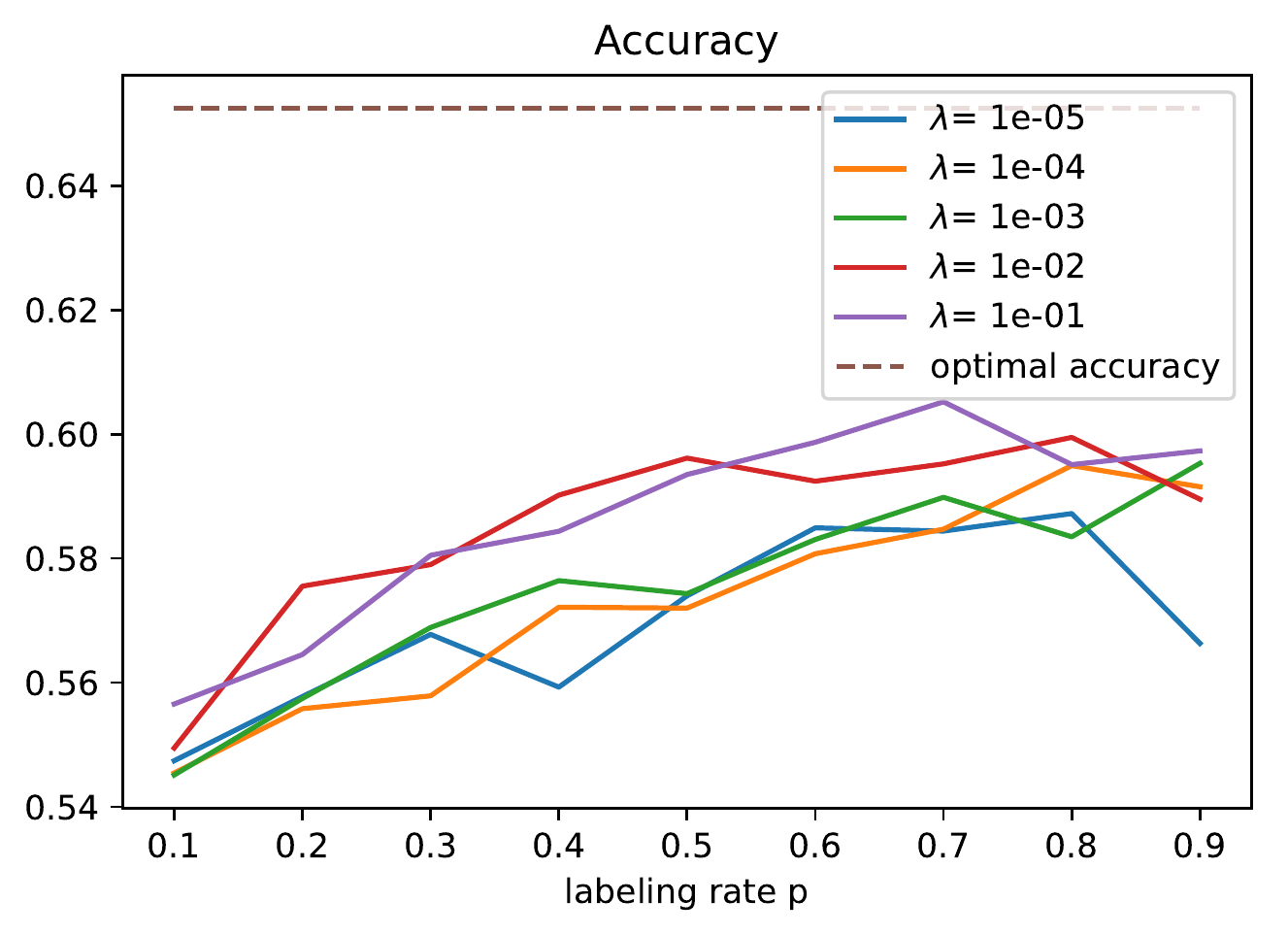}
	\vspace*{-3mm}
	\caption{The classification accuracy for chain-structured data.}
	\label{fig:accuracy_chain_8}
	\vspace*{-3mm}
\end{figure}
\begin{figure}[htbp]
	\includegraphics[width=1\columnwidth]{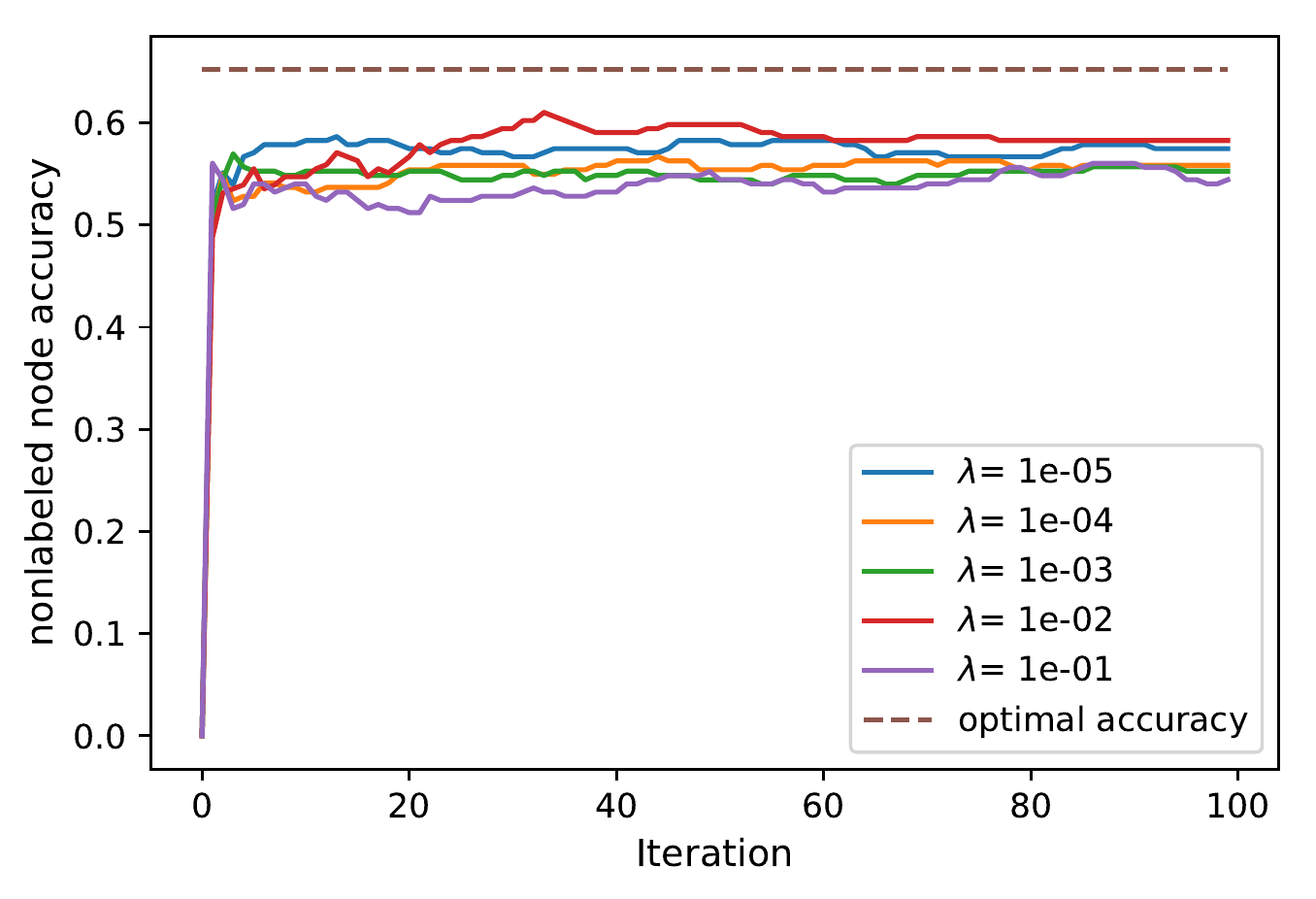}
	\vspace*{-3mm}
	\caption{The convergence rate of Alg. \ref{alg:ADMM} for chain structured data and labeling rate $p=0.4$.}
	\label{fig:convergence_chain8}
	\vspace*{-3mm}
\end{figure}

In Fig. \ref{fig:convergence_chain8}, 
we plot the accuracy (cf. \eqref{equ_exp_accu}) as a function of the number of iterations used in Alg. \ref{alg:ADMM} for varying 
lnLasso parameter $\lambda$ and fixed labeling rate $p\!=\!4/10$. Fig. \ref{fig:convergence_chain8} illustrates that, for a larger value of $\lambda$, e.g. $\lambda = 10^{-1}$ or $10^{-2}$, Alg. \ref{alg:ADMM} tends to deliver a classifier with better accuracy than that of smaller $\lambda$, e.g. $\lambda = 10^{-4}$ or $10^{-5}$. This proves that taking into account the network structure is beneficial to classify a networked data.
Moreover, it is shown in Fig. \ref{fig:convergence_chain8} that the accuracies do not improve after few iterations. This implies that lnLasso can yields a reasonable accuracy after a few number of iterations.


\subsection{Grid}
In the second experiment, we consider a grid graph with $N= 20 \!\times\! 20 \!=\! 400$ nodes. The graph is partitioned into $4$ clusters which are 
grid graphs of size $10 \times 10$. Similar to the chain, the edge weights $\gweight_{ij} =100$ if nodes $i$ and $j$ belong to the same cluster and $\gweight_{ij} =1$ otherwise. The average accuracies are plotted in Fig. \ref{fig:accuracy_grid4} which also shows that the accuracy increases with the labeling rate $p$. We also plot the accuracy (cf. \eqref{equ_exp_accu}) over iterations of Alg. \ref{alg:ADMM} for different values of $\lambda$ with $p=0.4$ in Fig. \ref{fig:convergence_grid4}. 
\begin{figure}[htbp]
	\includegraphics[width=1\columnwidth]{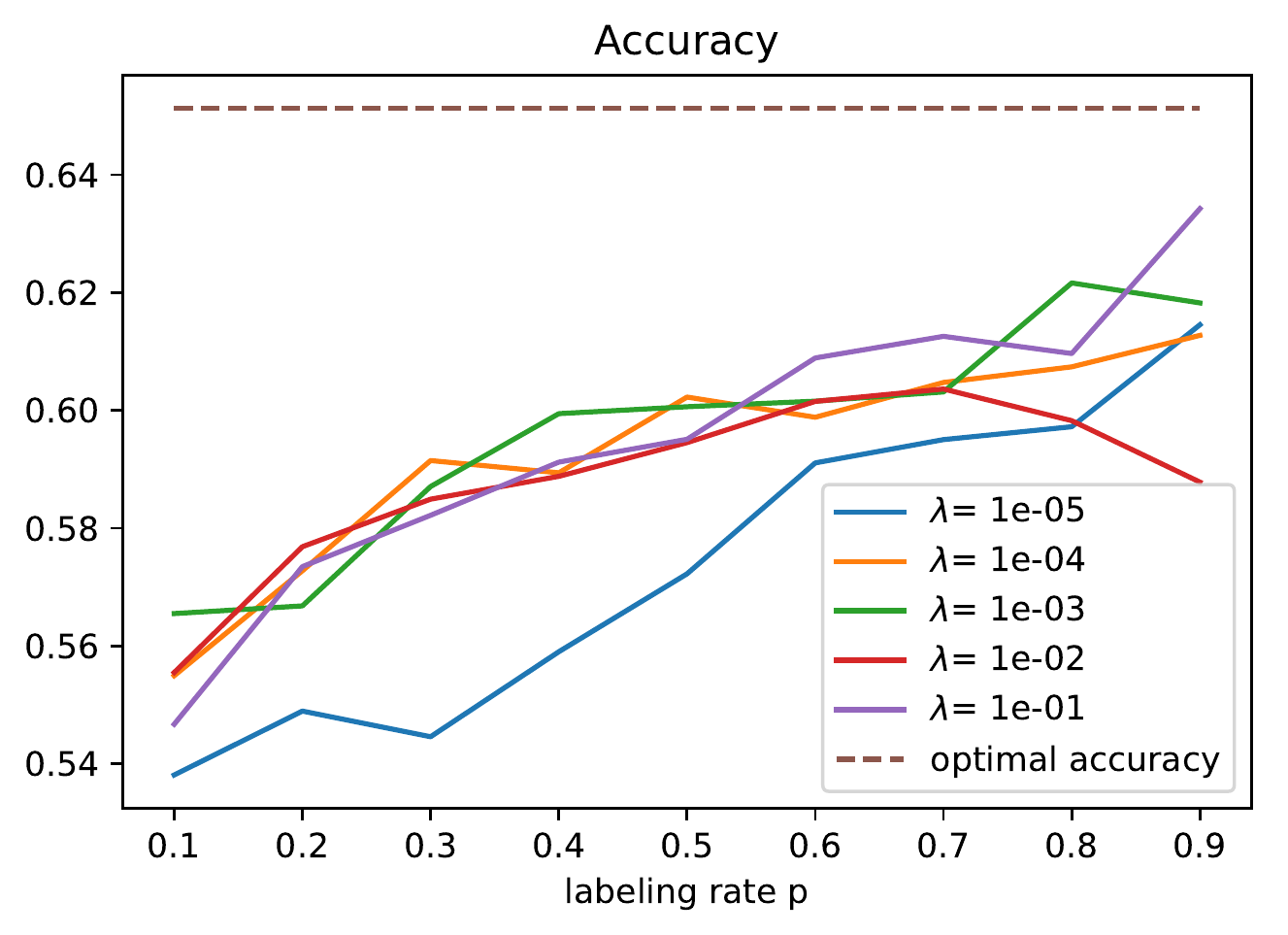}
	\vspace*{-3mm}
	\caption{The classification accuracy for grid-structured data.}
	\label{fig:accuracy_grid4}
	\vspace*{-3mm}
\end{figure}	

\begin{figure}[htbp]
	\includegraphics[width=1\columnwidth]{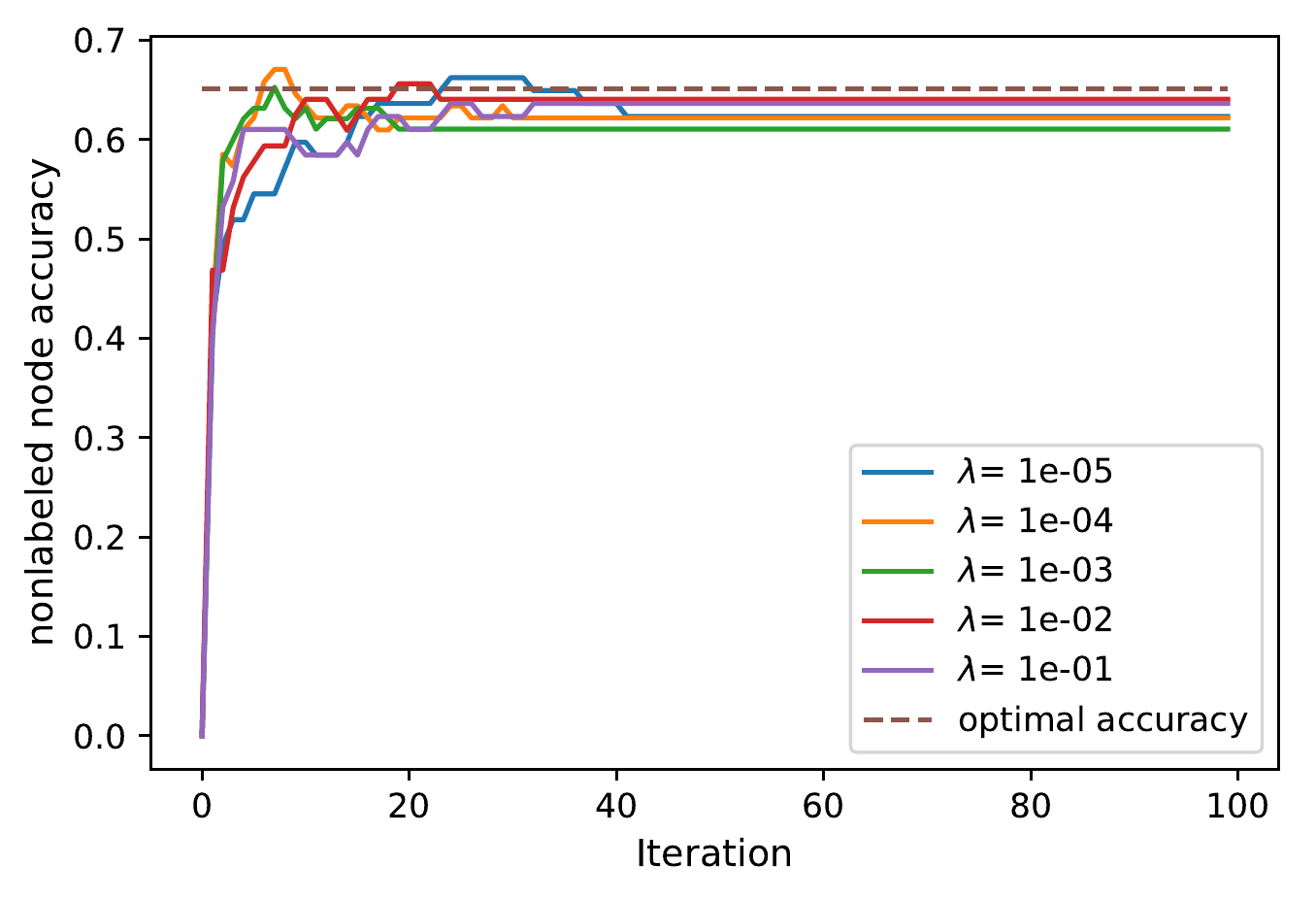}
	\vspace*{-3mm}
	\caption{The convergence rate of Alg. \ref{alg:ADMM} for grid structured data and labeling rate $p=0.4$.}
	\label{fig:convergence_grid4}
	\vspace*{-3mm}
\end{figure}

%

\bibliographystyle{IEEEbib}
\bibliography{SLPBib,tf-zentral}

\end{document}